\documentclass{article} 
\usepackage{iclr2027_conference,times}

\usepackage{graphicx} 
\usepackage{wrapfig}
\usepackage{subcaption}
\usepackage{amsmath,amssymb,amsthm}
\usepackage{booktabs}

\newtheorem{theorem}{Theorem}

\newtheorem{lemma}{Lemma}

\theoremstyle{definition}
\newtheorem{definition}{Definition}
\newtheorem{assumption}{Assumption}

\newcommand{\bfa}{\boldsymbol{a}}

\newcommand{\bfe}{\boldsymbol{e}}
\newcommand{\bff}{\boldsymbol{f}}
\newcommand{\bfg}{\boldsymbol{g}}
\newcommand{\bfh}{\boldsymbol{h}}

\newcommand{\bfq}{\boldsymbol{q}}

\newcommand{\bfs}{\boldsymbol{s}}

\newcommand{\bfv}{\boldsymbol{v}}

\newcommand{\bfx}{\boldsymbol{x}}
\newcommand{\bfy}{\boldsymbol{y}}

\newcommand{\bfA}{\boldsymbol{A}}

\newcommand{\bfI}{\boldsymbol{I}}
\newcommand{\bfJ}{\boldsymbol{J}}

\newcommand{\bfQ}{\boldsymbol{Q}}

\newcommand{\bfepsilon}{\boldsymbol{\epsilon}}

\newcommand{\bbR}{\mathbb{R}}

\newcommand{\cA}{\mathcal{A}}
\newcommand{\cB}{\mathcal{B}}

\newcommand{\cD}{\mathcal{D}}

\newcommand{\cL}{\mathcal{L}}
\newcommand{\cM}{\mathcal{M}}
\newcommand{\cN}{\mathcal{N}}
\newcommand{\cO}{\mathcal{O}}

\newcommand{\cU}{\mathcal{U}}
\newcommand{\cV}{\mathcal{V}}

\newcommand{\T}{{\!\top\!}}

\usepackage{hyperref}
\usepackage{url}

\title{Kinks vs. Smoothness: Identifiability of Real Analytic nICA for Laplace-like Sources}

\author{Isaac Manring, Kejun Huang \\
Department of Computer Science\\
University of Florida\\
Gainesville, FL 32611, USA \\
\texttt{\{imanring, kejun.huang\}@ufl.edu}
}

\iclrfinalcopy 
\begin{document}

\maketitle

\begin{abstract}
Many machine learning systems try to explain complex data - like images or financial time series - in terms of hidden, independent factors that generated them. Recovering the true underlying factors, rather than some scrambled version of them, is the central challenge of nonlinear Independent Component Analysis (nICA). We prove identifiability (exact recovery) up to trivial ambiguities for real analytic generating functions when source probability density functions have a finite number of discontinuities in the first derivative. The Laplace distribution is the most prominent example satisfying this assumption. Our proof relies on the contrast between kinks in the source distribution and the smoothness of real analytic functions. Real analytic functions comprise a broad class of generating mechanisms, and can be approximated with Normalizing Flows or Variational Autoencoders with standard activation functions (e.g., tanh, softplus, GELU), so our result applies with minimal changes to existing training pipelines. We perform experiments on real and synthetic data with both Normalizing Flows and Variational Autoencoders demonstrating their identifiability properties. In experiments on CelebA data we recover several interpretable latent factors controlling unique attributes across the dataset.
\end{abstract}

\section{Introduction}

Nonlinear Independent Component Analysis (nICA) has foundational theoretical implications for unsupervised and representation learning. For observations $\bfx^{(i)} \in \bbR^m$, it seeks to recover statistically independent generative factors (also called sources) $\bfs^{(i)} \in \bbR^k$, such that
\begin{equation} \label{eq:generating}
    \bfx^{(i)} = \bff(\bfs^{(i)}) +\bfepsilon^{(i)} \leftrightarrow \bfg(\bfx^{(i)} - \bfepsilon^{(i)}) = \bfs^{(i)} \ ,
\end{equation}
where $\bff$ is a nonlinear diffeomorphism, with its inverse $\bfg$, and $\bfepsilon^{(i)}$ is noise (see formal definitions in \autoref{sect:prelim}). When $m=k$ the noise component is generally dropped. This model is in general not able to recover the generative factors $\bfs^{(i)}$ because of its non-identifiability (see Definition \ref{def:identifiable} for a formal definition). One of the central goals of nICA is to guarantee the recovery of the true latent factors given a set of assumptions.
These guarantees are essential for interpretable latent representations. Additionally, they are useful for training stability and reproducibility \citep{hyvarinen2024identifiability, KivvaIdentifiability2022}. Last but not least, \citet{damourUnderspecification2022} point to non-identifiability as a cause of the drop in performance of deep learning systems at deployment.

In this paper, we prove identifiability of nICA without auxiliary observations when the source distribution has a discontinuity in the first derivative (for example Laplacian) and $\bff$ and $\bfg$ are real analytic. This is a substantial contribution to nICA theory as the assumption of analyticity is mild. As a caveat we note in \autoref{sect:limitations} that while there are several distributions on compact supports that are theoretically identifiable according to \autoref{thm:id_multi_point}, they are not practically identifiable due to the universal approximation ability of real analytic diffeomorphisms on compact sets.

Neural networks can easily be constrained to be real analytic by requiring the activation functions to be real analytic. Thus training for our method which we dub RAD (Real Analytic Decoders) requires minimal changes to standard techniques, enabling easy adoption.
We demonstrate the validity of the theoretical results through extensive simulation experiments. RAD also demonstrates good stability in training on real world data and recovers several interpretable latent variables for the CelebA dataset \citep{liu2015faceattributes}. 

\section{Related Work}

While nICA is still developing, approaches to guarantee identifiability generally fall into two categories. The first utilizes auxiliary observations such as class labels, time-series indices, or the preceding observation for auto-correlated data \citep{khemakhem2020variational,hyvarinen2019nonlinear, Sorrenson2020Disentanglement}. The sources are assumed to be conditionally independent given the auxiliary variable. \citet{halva2020hidden} extended this using a Hidden Markov Model so that the auxiliary variables can be unobserved too.

The second approach does not assume any auxiliary information, but restricts the function class of $\bff$. \citet{zheng2022identifiability, zheng2023generalizing} proved identifiability assuming that the Jacobian of $\bff$ had a certain sparsity pattern. \citet{nguyen2025diverseinfluencecomponentanalysis} proposed the idea that in order to recover $\bfs^{(i)}$, each source should control a sufficiently distinct aspect of the observations. Mathematically, this was formulated as a sufficiently scattered condition on the columns of the Jacobian of $\bff$. This requires the observations to be in a much higher dimension than the sources.
In a similar vein \citet{buchholz2022function} showed that if the Jacobian was a scaled orthonormal matrix at every point (a conformal map), then identifiability could be achieved.
Drawing from optimal transport theory, \citet{huang2021convex} trained a unique neural network to recover the sources. The parameterization ensured that $\bff$ was the gradient of a convex function (its Jacobian was symmetric positive semi-definite).
\citet{yang2022nonlinear} shows identifiability for volume preserving flows when there are at least two distinct but overlapping classes in the data.

In the most comparable work to ours, \citet{KivvaIdentifiability2022} considered real analytic mixture distributions while constraining the form of $\bff$ to be piecewise affine. In several theorems they prove identifiability of increasing strength under assumptions of increasing strength. The result most comparable to \autoref{thm:id_multi_point} shows that with a prior of independent Gaussian Mixture Models (GMMs) and a piecewise affine $\bff$, $\bfs_i$ can be recovered up to permutation, scaling, and translation. While our proof does not resemble the proof presented in \citet{KivvaIdentifiability2022}, the assumptions are interestingly duals of each other: we assume real analytic $\bff$ and non-differentiable source distribution, while they assume piecewise affine $\bff$ and real analytic source distribution.

While at first glance RAD may seem similar to Sparse Autoencoders (SAE) utilized in transformer interpretability because of the use of a Laplace source distribution, there are several important distinctions. First, SAEs generally only have three layers: linear encoder, ReLU, linear decoder \citep{HubenSparse2024}. RAD in general has many more than this. Second, the latent dimension in SAEs is higher than the observation dimension, which is not allowed in RAD. Thus SAEs form a distinct line of research.

\section{Preliminaries} \label{sect:prelim}

To begin, we introduce some preliminary definitions. We require $\bff$ and $\bfg$ to be real analytic diffeomorphisms.

\begin{definition}[Diffeomorphism \citep{Boumal_2023}] \label{def:diffeomorphism}
A diffeomorphism is a bijective map $\bff: \cU \rightarrow \cV$, where $\cU \subseteq \bbR^k$, $\cV \subseteq \bbR^m$
are open sets such that both $\bff$ and $\bff^{-1}=\bfg$ are smooth.
\end{definition}

Note that diffeomorphisms do not require $m=k$. If $m > k$ (the data are higher dimensional than the generative factors), we simply require the data to live in a lower dimensional embedded submanifold (see Definition \ref{def:ra_submanifold}).

\begin{definition} [Real Analytic Function \citep{KrantzAnalytic2002}] \label{def:real_analytic}
A function $\bff: \cU \subseteq R^m \rightarrow R^n$ is real analytic if for every $\bfx_0 \in \cU$ each component $\bff_i$ admits a convergent Taylor expansion about $\bfx_0$ equal to $\bff_i$ in some neighborhood around $\bfx_0$.
\end{definition}

The set of real analytic functions $C^\omega$ is a subset of infinitely differentiable functions $C^\infty$. Many familiar functions, such as polynomials, $\exp(\cdot)$, and $\cos(\cdot)$, are real analytic. $C^\omega$ is closed under composition and inversion.
If the functions $u,v$ are real analytic at point $x_0$ then $u(x) + v(x)$ and $u(x)v(x)$ are real analytic at $x_0$. Additionally, $\frac{1}{u(x)}$ is real analytic at $x_0$ if $u(x_0) \neq 0$.

As mentioned earlier, in the case when $m > k$ we adopt the assumption that the data approximately lie on a low-dimensional submanifold embedded in a higher dimension. This is the standard assumption of representation learning methods, such as VAEs, when the embedded dimension is lower than the observation dimension \citep{goodfellow2016deep}. Because $\bff$ is real analytic this submanifold must be real analytic also.
\begin{definition}[Real Analytic Submanifold \citep{KrantzAnalytic2002}] \label{def:ra_submanifold}
The set $\cM \subseteq \bbR^m$ is a $k$-dimensional real analytic submanifold if 
for each $p\in\cM$ there exist an open $\cV$ with $p \in \cV \subseteq \bbR^m$, a convex open $\cU \subseteq \bbR^k$, and real analytic maps $\bfg: \cV \rightarrow \cU$, $\bff: \cU \rightarrow \cV$ such that
$\cM \cap \cV = \operatorname{image} \bff$ and $\bfg \circ \bff$ is the identity on $\cU$. 
\end{definition}

This assumption implies that there is redundant information in the observations. For example, in an image each pixel value is not free to take on any value as it must be similar to those around it for the image to convey meaningful information. This is why compression is possible.
Because in practice $\bfx^{(i)}$ almost never exactly lie on a submanifold, it is assumed that they are noisy. 
We define $p_{\bff,\bfg}(\bfx)$ as the probability density of $\bfx$ given the model in \autoref{eq:generating} parameterized by $\bff$ and $\bfg$ with known source distribution $\rho$. 
We provide the standard definition of identifiability up to an ambiguity class below.

\begin{definition}[Identifiability] \label{def:identifiable}
The model in \autoref{eq:generating} is said to be identifiable up to equivalence class $\simeq$ if
\begin{equation*}
\forall 
\bfx \in \bbR^m,  p_{\tilde{\bff},\tilde{\bfg}}(\bfx) = p_{\bff,\bfg}(\bfx) \implies \bfg \simeq \tilde{\bfg}.
\end{equation*}
\end{definition}

\section{Main Result} \label{sect:main}

In this section we present the identifiability proof. First, using standard techniques, we transform the model in \autoref{eq:generating} to be in terms of learning a self-map of the sources $\bfh = \tilde{\bfg} \circ \bff$. This step also eliminates the noise from $\bfepsilon$ in the case of $m > k$ and requires Assumption \ref{asm:noise_density}, which comes directly from \citet{khemakhem2020variational}. 
Then in \autoref{thm:id_multi_point} we prove that $\bfh$ must be a signed permutation (scaling and translation are embedded in the knowledge of the source distribution).

\begin{assumption} [\citep{khemakhem2020variational}] \label{asm:noise_density}
The set $\{ \bfx \in \bbR^m| \varphi_\epsilon (\bfx) = 0 \}$ has measure zero, where $\varphi_\epsilon$ is the characteristic function of the density of $\epsilon$.
\end{assumption}

\begin{lemma}
Suppose that $\epsilon$ satisfies Assumption \ref{asm:noise_density} and we learn real analytic $\tilde \bff$ and $\tilde{\bfg}$ such that $\tilde{\bfg} \circ \tilde \bff$ is the identity and 
$\forall \bfx \in \bbR^m, p_{\tilde{\bff},\tilde{\bfg}}(\bfx) = p_{\bff,\bfg}(\bfx)$. Then $\bfh = \tilde \bfg \circ \bff$ is a real analytic diffeomorphism and $\bfh(\bfs) \sim \rho$.
\end{lemma}
\begin{proof}
The first part of this proof is relatively trivial.
From Definition 2.7.4 of \citep{KrantzAnalytic2002} it is clear that the composition $\tilde{\bfg} \circ \bff$ is real analytic in $\bbR^k$. It is also clear from Definition \ref{def:diffeomorphism} that the composition of diffeomorphisms is a diffeomorphism.

The second part removes the noise $\bfepsilon$. Using Assumption \ref{asm:noise_density}, \citet{khemakhem2020variational} provided the standard proof for this in Appendix B.2.2 Step I using a deconvolution argument.
\end{proof}

We utilize the contrast between real analytic generating functions and non-differentiable source distributions to prove identifiability. The non-differentiability of the source distribution $\rho$ must satisfy the following assumption.

\begin{assumption} \label{asm:mult_non_diff_pdf}
$\rho$ is a probability density function with connected support $\cD \subseteq \bbR$. $\log \rho$ can be partitioned into a finite number of real analytic functions $\alpha_0, \dots, \alpha_d$ at distinct breakpoints $\beta_1, \dots, \beta_d$ for $d \geq 1$. The following must hold,
\begin{enumerate}
    \item $\alpha_{i-1}'(\beta_i) \neq \alpha_{i}'(\beta_i) \ \forall i=1,\dots, d$
    \item $\beta_i - \beta_{i-1} > 0 \ \forall i=1,\dots, d+1$ with the convention that $\beta_0=\inf \cD$ and $\beta_{d+1} = \sup \cD$
    \item For $d$ odd, there are analytic continuations of $\alpha_{\frac{d-1}{2}}$ and $\alpha_{\frac{d+1}{2}}$ on an open neighborhood around $\beta_{\frac{d+1}{2}}$. Similarly for $d$ even, there are analytic continuations of $\alpha_{\frac{d}{2}-1}$ and $\alpha_{\frac{d}{2}+1}$ onto the interval $[\beta_{\frac{d}{2}}, \beta_{\frac{d}{2}+1}]$.
\end{enumerate}
\end{assumption}

The last requirement in Assumption \ref{asm:mult_non_diff_pdf} is used in step 3 of the proof of \autoref{thm:id_multi_point} in order to use the Taylor expansion around a common point.
Note that although distributions of the form $\rho(x) \propto \exp(-|x|^\alpha), \alpha \in (0,1)$ satisfy requirements 1 and 2 of Assumption \ref{asm:mult_non_diff_pdf}, they fail on requirement 3.
Several well known distributions satisfy this assumption including the Laplace (also called double exponential) distribution and the triangular distribution.
The probability density function of the Laplace distribution is given by
\begin{equation*}
\rho_{\cL}(s) = \frac{1}{2b}\exp\left(-\frac{1}{b} |s-\mu|\right) ,
\end{equation*}
for location and scale parameters $\mu$ and $b$. The Laplace is well known because its maximum likelihood estimate uses the $\ell_1$ norm penalty which induces sparsity while allowing outliers. Derivatives of the Laplace distribution like Laplace Mixture Model, Log Laplace, and asymmetric Laplace also satisfy Assumption \ref{asm:mult_non_diff_pdf}. A Spike-and-Slab distribution built from Laplace distributions to induce sparsity as used by \citet{moran2022identifiable} also satisfies the assumption.

The standard triangular distribution describes the probability density of the sum of two standard uniform random variables. However, in its general form it can also be asymmetric,
\begin{equation*}
\rho_\Delta(x) = 
\begin{cases}
    0, & x < a \\
    \frac{2(x-a)}{(b-a)(c-a)}, & a \leq x \leq c \\
    \frac{2(b-x)}{(b-a)(b-c)}, & c \leq x \leq b \\
    0, & b< x
\end{cases}
\end{equation*}

While this distribution satisfies Assumption \ref{asm:mult_non_diff_pdf}, it may be difficult to use in practice (see \autoref{sect:limitations}). Exponential of many piecewise defined polynomials and other non-standard distributions also satisfy Assumption \ref{asm:mult_non_diff_pdf}.
The essential characteristic is the discontinuity in the derivative of the probability density function which implies a sharp change in the probability law of the source. Larger discontinuities are harder to approximate with real analytic functions and thus may make recovery easier. We now present our main result.

\begin{theorem} \label{thm:id_multi_point}
Suppose $\bfh$ is a real analytic diffeomorphism, $\bfs$ is iid according to $\rho$, and $\bfh(\bfs)$ is also iid according to $\rho$. If $\rho$ satisfies Assumption \ref{asm:mult_non_diff_pdf}, then $\bfh$ must be the identity up to signed permutation and translation.
\end{theorem}
\begin{proof}

The proof proceeds in three steps. In the first step we show that $\sum_i \log \rho(\bfs_i) - \sum_i \log \rho( \bfh_i(\bfs))$ must be real analytic. In the second step, we show that because the kinks in $\log \rho (\bfs_i)$ must be canceled out by the kinks in $\log \rho( \bfh_{\pi_1(i)}(\bfs))$ we can factorize $\bfh_{\pi_1(i)}$ as in \autoref{eq:h_factorize}. Finally in step three, using Taylor expansion we show that $r_i(\bfs) = \pm 1$ in \autoref{eq:h_factorize}.

Because $\bfs \sim \rho$ and $\bfh(\bfs) \sim \rho$, the change of variables formula gives
\begin{align*}
\prod_i\rho(\bfs_i) = |\det \bfJ_h(\bfs)| \prod_i\rho(\bfh_i(\bfs)) \\
\sum_i \log \rho(\bfs_i) - \sum_i \log \rho( \bfh_i(\bfs)) = \log |\det \bfJ_h(\bfs)|
\end{align*}
Because $\bfh$ is invertible $\det \bfJ_h(\bfs) \neq 0$. Further due to the continuity of $\det \bfJ_h(\bfs)$ it always has the same sign. Therefore, the absolute value reduces to a constant multiple. Since $\bfh$ is real analytic, $\bfJ_h(\bfs)$ is also real analytic. The determinant is a polynomial function, which is real analytic, as is the logarithm. Therefore $\log \det \bfJ_h(\bfs)$ is real analytic. This implies that $\gamma(\bfs)=\sum_i \log \rho(\bfs_i) - \sum_i \log \rho( \bfh_i(\bfs))$ must also be real analytic.

Let $\cA_{ij}=\{\bfs | \bfs_i=\beta_j\}$ and $\cB_{lm} = \{\bfs | \bfh_l(\bfs)=\beta_m \}$. At $\cA_{ij}$ there is a non-differentiable point in $\log \rho(\bfs)$ as you vary $\bfs_i$. It must be canceled out by $\sum_i \log \rho( \bfh_i(\bfs))$ to maintain the real analytic property. Thus $\cA_{ij} \subset \bigcup_{l,m} \cB_{lm}$ because there can only be non-analytic points in $\sum_i \log \rho( \bfh_i(\bfs))$ around $\bigcup_{l,m} \cB_{lm}$ because $\bfh$ is real analytic.

Suppose that $\nexists \ l,m, $ such that $\bfh_l \equiv \beta_m$ on an open neighborhood of $\bfs \in \cA_{ij}$. 
For real analytic $\bfh_l - \beta_m$ the measure of the zero-set is zero \citep{mityagin2015zerosetrealanalytic}.
Thus $\cB_{lm} \cap \cA_{ij}$ has a measure of zero. But if all $\cB_{lm}$ have measure zero in $\cA_{ij}$ then we reach a contradiction in the measure of $\cA_{ij}$. So $\exists \ l,m$ such that $\bfh_l \equiv  \beta_m$ on an open neighborhood of $\cA_{ij}$. By the real analytic identity theorem (Corollary 1.2.6 in \citep{KrantzAnalytic2002}) $\bfh_l \equiv  \beta_m$ on all of $\cA_{ij}$.

Because $\bfh_l$ is constant on $\cA_{ij}$ the directional derivatives $\bfv^\T \nabla \bfh_l(\bfs) = 0, \forall \bfs \in \cA_{ij}, \bfv_i=0$ which implies $\nabla \bfh_l(\bfs) \in \text{span}\{\bfe_i\}$.
Suppose that both $\bfh_l-\beta_j$ and $\bfh_k-\beta_j$ vanish on $\cA_{ij}$, then $\nabla \bfh_l(\bfs), \nabla \bfh_k(\bfs) \in \text{span}\{\bfe_i\}, \forall \bfs \in \cA_{ij}$. However, this contradicts the inverse function theorem which says that the Jacobian of $\bfh$ is invertible.
Therefore $\cA_{ij} = \cB_{\pi(i,j)}$ for permutations $\pi$.

$\pi$ can be further simplified. For $j\neq j'$, $\varnothing = \cA_{ij} \cap \cA_{ij'} = \cB_{\pi(i,j)} \cap \cB_{\pi(i,j')}$. However, $\cB_{lm} \cap \cB_{l'm'}=\varnothing$ only when $l=l'$ and $m' \neq m$.
Therefore, we can factorize $\pi(i,j)$ into $\pi_1(i),\pi_2(j)$. This means that all kinks in $\rho(\bfs_i)$ are canceled by a single dimension of $\bfh$. For simplicity we decompose $\bfs$ into $u = \bfs_i$ and $\bfv$ be the rest of $\bfs$.
For a fixed $\bfv$, $\bfh$ must pass through each kink exactly once or cancellation will not be achieved.
Therefore because of continuity in $\bfh_{\pi_1(i)}(u,\bfv)$ with respect to $u$ $\pi_2$ can be reduced to the identity or the reverse order permutation.

Because $\bfh_{\pi_1(i)} - \beta_{\pi_2(j)}$ vanishes precisely when $u - \beta_j$ vanishes and $\nabla \bfh_{\pi_1(i)} \neq \mathbf{0}$, we can use the Weierstrass Preparation Theorem (Theorem 6.1.3 \citep{KrantzAnalytic2002}) to factorize $\bfh_{\pi_1(i)}- \beta_{\pi_2(j)}$
\begin{equation}\label{eq:h_factorize}
\bfh_{\pi_1(i)}(u, \bfv)- \beta_{\pi_2(j)}= r_i(u,\bfv)(u - \beta_j)
\end{equation}
where $r_i(u,\bfv) \neq 0$ and is real analytic. Note here the Weierstrass polynomial is of degree 1 because $\nabla \bfh_i \neq \mathbf{0}$.

Step 3 utilizes the uniqueness of the Taylor coefficients of $\gamma_i(u) = \log \rho(\bfh_{\pi_1(i)}(u,\bfv)) - \log \rho(u)$ by expanding multiple segments of $\gamma_i$ around a shared center point to prove that $r_i(\bfs) \equiv \pm 1$. Due to space constraints, we leave the full details for \autoref{app:step_3}.

Using $r_i \equiv \pm 1$ in \autoref{eq:h_factorize}, we see that $\bfh$ is a signed permutation with a potential translation.

\end{proof}

\subsection{Model Architecture}

To learn real analytic diffeomorphisms we consider two well known neural network architectures: normalizing flows \citep{Kobyzev_2021} and variational autoencoders \citep{kingma2019VAE}.
Since affine transformations are real analytic, constraining a neural network to be real analytic only requires the activation functions to also be real analytic. 

\begin{wrapfigure}{R}{0.5\textwidth}
\vskip -0.2in
\centering
\centerline{\includegraphics[width=0.5\columnwidth]{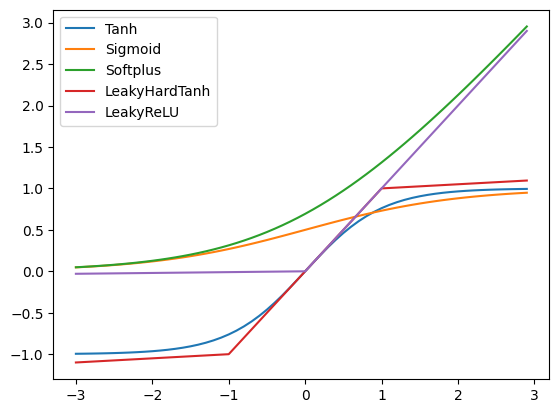}}
\caption{Common monotonic activation functions}
\label{fig:act_fun}
\vskip -0.95in
\end{wrapfigure}

Normalizing Flows require invertible activation functions, which means that they must be monotonic. We classify several popular monotonic activation functions shown in \autoref{fig:act_fun} by their membership in $C^\omega$. The activation functions below are monotonic and real analytic,
\begin{enumerate}
    \item $\operatorname{Tanh}(x) = \frac{e^x - e^{-x}}{e^x + e^{-x}}$
    \item $\operatorname{Sigmoid}(x) = \frac{1}{1 + e^{-x}}$
    \item $\operatorname{Softplus}(x) = \log ( 1 + e^x)$
\end{enumerate}
The following are monotonic but not real analytic,
\begin{enumerate}
    \item $\operatorname{LeakyReLU}(x) = \begin{cases}
        x, & x \geq 0 \\
        \alpha x, & x < 0
    \end{cases}$
    \item $\operatorname{LeakyHardTanh}(x) = \begin{cases}
        \alpha x - 1, & x < -1 \\
        x, & |x| \leq 1 \\
        \beta x + 1, & x > 1 \\
    \end{cases}$
\end{enumerate}

VAEs do not require invertible activation functions and so may also use more modern alternatives like $\operatorname{Swish}$ \citep{ramachandran2017searchingactivationfunctions} or $\operatorname{GELU}$ \citep{hendrycks2023gaussianerrorlinearunits}, both of which are real analytic.

\subsection{Assumption Tightness}

In this section we show that both assumptions of real analyticity and non-differentiable source distributions are necessary for the identifiability in \autoref{thm:id_multi_point}. Some of the most well known nICA identifiability counter-examples are the so called Measure Preserving Automorphisms (MPAs) \citep{HYVARINEN1999429}. For any continuous source distribution $\rho$ with mild regularity assumptions there exists a transformation $\bfq$ such that $\bfq(\bfs) \sim \cN(\mathbf{0}, \bfI)$. This $\bfq$ can be constructed by $q = \Phi^{-1} \circ R$ where $\Phi$ is the Gaussian Cumulative Distribution Function (CDF) and $R$ is the CDF of $\rho$.
Since the normal distribution is invariant under orthogonal transformation, $\bfQ \bfq(\bfs) \sim \cN(\mathbf{0}, \bfI)$. Finally the sources can be transformed back to their original distribution so that $\bfq^{-1} \left(\bfQ \bfq(\bfs)\right)$ has the same distribution as $\bfs$.

This provides a clear counter-example when we relax the requirement of $\bfh$ to be real analytic. However, $\bfq$ cannot be real analytic if $\rho$ is not real analytic because $R$ would not be real analytic. MPAs also provide a counter-example when $\bfs$ is distributed according to a real analytic distribution. This includes obvious distributions like Gaussian, but also when there is a discontinuity in the PDF between the 0 and non-zero parts such as for the Gamma or Beta distributions.

\subsection{Limitations} \label{sect:limitations} 

The Whitney Approximation Theorem \citep{KrantzAnalytic2002} states that real analytic functions can approximate any $C^r$ function for any $r \geq 0$ to arbitrary precision on a compact set (closed and bounded in euclidean space). Further these approximation results can be attained by neural networks. \citet{IshikawaApproximation2023} provides general results on the universal approximation ability of invertible neural networks (Normalizing Flows). With our architecture we impose the constraint that $\bff \in C^\omega \subset C^\infty$. So \citet{IshikawaApproximation2023} demonstrates that our method is a universal function and universal density approximator. \citet{PuthawalaApproximation2022} provides universal approximation results for injective flows where the observation dimension is greater than the source dimension.

Proving identifiability for universal function approximators is a two edged sword. On the one hand, they can approximate any generating function, but on the other hand they can also approximate identifiability counter-examples like Measure Preserving Automorphisms (MPAs). Therefore, while \autoref{thm:id_multi_point} indicates that RAD is identifiable for compactly supported source distributions, training a real analytic approximator is very difficult because it could approximate an MPA. Note that this does not apply to source distributions on non-compact supports like the Laplace.

In particular, a model with triangular or truncated Laplace source distribution, where the tails are cut off, while theoretically identifiable, may be practically non-identifiable. 
In \autoref{sect:syn_exp} we tested these source distributions. \autoref{tab:dist_comp} indicates that we were not able to recover the sources, confirming our hypothesis about the universal approximation problem.

Incidentally, this is also a limitation of \citep{KivvaIdentifiability2022} due to the universal approximation ability of piecewise affine functions.
One cannot prove identifiability for a class of universal approximators and then rely on those functions to approximate any function not in the class. Rather the practitioner must know something substantive about the generating function in order to achieve identifiability.

\section{Experiment}

We run three experiments (Synthetic, Yahoo Stock, and CelebA) to verify the theory. The synthetic experiments have ground truth labels and thus are the primary means of identifiability verification. We utilize the standard Mean Correlation Coefficient (MCC) metric, which is the mean absolute correlation between learned and true sources after matching. Full details on experiments are in \autoref{app:exp_details}.

\subsection{Synthetic Data} \label{sect:syn_exp}

In these experiments we generate synthetic data and train Normalizing Flows to recover the ground truth factors, in order to confirm \autoref{thm:id_multi_point}. Experiments on alternate generating functions are recorded in \autoref{app:exp_details}.

The sources were independent and identically generated from source distribution $\rho$. We used two nonlinear mixture functions which we term mixture A and mixture B.

\textbf{Mixture A}: $\bff$ is a post-nonlinear function $\bfx^{(i)} = \bfa(\bfA \bfs^{(i)})$ where $\bfa$ is applied in a component-wise manner. For each component $j$ we randomly generated the function $\bff_j(y) = \frac{1}{4} \sin(y+2 \pi w_i) + \frac{y}{3}$ where $w_i \sim U(0,1)$. 
For each trial a new random normal mixing matrix $\bfA$ was generated. 
To make training easier, we shrank the log singular values by 30\%. This shrinks very high and very low singular values which were hard for the neural networks to learn. The training would significantly slow down when singular values were too high or low. Then we normalized by $\sqrt{k}$ so that the range of $\bfy$ would fall on the most nonlinear part of $f_i$.

\textbf{Mixture B}: $\bff$ is a 3-layer invertible neural network with $\operatorname{Softplus}$ activations. The 3 weight matrices were random normal matrices. Again, for training ease we shrank the log singular values of the weight matrices by 40\%.

In the experiments below we train Normalizing flows with 6 dense invertible layers. Unless otherwise stated, the final layer has no activation applied. For each experiment we trained on 5 different randomly generated datasets using the same seeds across experiments. We calculate Mean Correlation Coefficient (MCC) and report mean and standard deviation across trials.

\begin{table}[tb]
    \centering
    \caption{Average MCC $\pm$ standard deviation for Laplace source model with different activation functions in the Normalizing Flow. Real analytic activation functions achieved substantially higher MCCs than non-real analytic activation functions.}
    \label{tab:act_fun_comp}
    \begin{tabular}{lcccc}
        \toprule
        & \multicolumn{2}{c}{$k=5$} & \multicolumn{2}{c}{$k=10$} \\
        \cmidrule(lr){2-3} \cmidrule(lr){4-5}
        Activation Function & Mixture A & Mixture B & Mixture A & Mixture B \\
        \midrule
        Tanh & 0.944 $\pm$ 0.017 & 0.945 $\pm$ 0.007 & 0.930 $\pm$ 0.023 & 0.920 $\pm$ 0.037 \\
        Sigmoid & 0.895 $\pm$ 0.030 & 0.925 $\pm$ 0.031 & 0.842 $\pm$ 0.085 & 0.895 $\pm$ 0.042 \\
        Softplus & 0.833 $\pm$ 0.075 & 0.916 $\pm$ 0.014 & 0.765 $\pm$ 0.110 & 0.859 $\pm$ 0.032 \\
        LeakyHardTanh & 0.431 $\pm$ 0.198  & 0.577 $\pm$ 0.081 & 0.448 $\pm$ 0.057 & 0.502 $\pm$ 0.030 \\
        LeakyReLU & 0.143 $\pm$ 0.159 & 0.067 $\pm$ 0.057 & 0.257 $\pm$ 0.107 & 0.192 $\pm$ 0.188 \\
        \bottomrule
    \end{tabular}
\end{table}

In the first experiment, we compare identifiability of a Laplacian source model for different activation functions for both $m=k=5$ and $m=k=10$ with $n=10,000$ samples. \autoref{tab:act_fun_comp} contains the results. Real analytic activation functions resulted in markedly higher MCCs for both Mixture A and Mixture B as \autoref{thm:id_multi_point} would indicate.

In the second experiment, we compare identifiability of the model with different source distributions.
We fixed the activation function as $\operatorname{Tanh}$ to keep $\bfg$ real analytic and because $\operatorname{Tanh}$ performed the best in the first experiment. Otherwise the architecture was the same. We considered higher dimensions $m=k=15,25$.
Interestingly we found that a lower sample of $n=5,000$ would still achieve high MCCs. 
We tested six different source distributions: Laplace, Laplace Mixture Model (LMM), Gaussian, Gamma, Triangular, and Truncated Laplace. For Laplace and Gaussian source distributions, we set the mean to 0 and scale to 1. For the Gamma source distribution we used $\alpha=5, \beta=5$ and added a $\operatorname{Softplus}$ activation to the final layer in $\tilde \bfg$ to ensure positive values. The LMM was a mixture of two Laplace distributions with scales 1 and means -0.5 and 0.5 with a $0.75,0.25$ weighting respectively. Because the standard triangular distribution has support on $[-1,1]$ we applied a $\operatorname{Tanh}$ activation to the final layer. Finally the truncated Laplace distribution was a Laplace distribution truncated to the support $[-2,2]$. So we applied a $\operatorname{Tanh}$ to the final layer and then multiplied by 2.
According to \autoref{thm:id_multi_point} the Laplace, LMM, Triangular, and truncated Laplace models are identifiable while the Gaussian and Gamma are not. 
However, as mentioned in \autoref{sect:limitations}, the Triangular, and truncated Laplace models have compactly supported source distributions and thus are not expected to be able to recover the sources.

\begin{table}[tb]
    \centering
    \caption{Average MCC $\pm$ standard deviation for different source distributions. The Laplace source distribution results in high MCCs even for higher dimensions.}
    \label{tab:dist_comp}
    \begin{tabular}{lcccc}
        \toprule
        & \multicolumn{2}{c}{$k=15$} & \multicolumn{2}{c}{$k=25$} \\
        \cmidrule(lr){2-3} \cmidrule(lr){4-5}
        Distribution & Mixture A & Mixture B & Mixture A & Mixture B \\
        \midrule
        Laplace & 0.965 $\pm$ 0.004 & 0.925 $\pm$ 0.015 & 0.953 $\pm$ 0.006 & 0.916 $\pm$ 0.037 \\
        LMM & 0.937 $\pm$ 0.021 & 0.919 $\pm$ 0.027 & 0.939 $\pm$ 0.013 & 0.904 $\pm$ 0.025 \\
        Gaussian & 0.514 $\pm$ 0.011 & 0.503 $\pm$ 0.034 & 0.417 $\pm$ 0.011 & 0.423 $\pm$  0.018 \\
        Gamma & 0.511 $\pm$ 0.029 & 0.442 $\pm$ 0.015 & 0.432 $\pm$ 0.021 & 0.399 $\pm$ 0.020 \\
        Triangular & 0.453 $\pm$ 0.010 & 0.461 $\pm$ 0.013 & 0.385 $\pm$ 0.005 & 0.395 $\pm$ 0.007 \\
        Trunc. Laplace & 0.484 $\pm$ 0.033 & 0.469 $\pm$ 0.016 & 0.400 $\pm$ 0.006 & 0.394 $\pm$ 0.008\\
        \bottomrule
    \end{tabular}
\end{table}

Results reported in \autoref{tab:dist_comp} show high MCCs for the Laplace and LMM models, and low MCCs for Gaussian, Gamma, Triangular, and Truncated Laplace models. The high MCCs of the LMM model demonstrate that even when there are multiple non-differentiable points the model is still identifiable. Finally we provide an example correlation plot between true and learned sources in \autoref{fig:tanh_v_lrelu} comparing an identifiable Laplace model to a non-identifiable Gamma model. 

\begin{figure}[ht]
    \centering
    \begin{subfigure}[t]{0.45\textwidth}
        \centering
        \includegraphics[width=\textwidth]{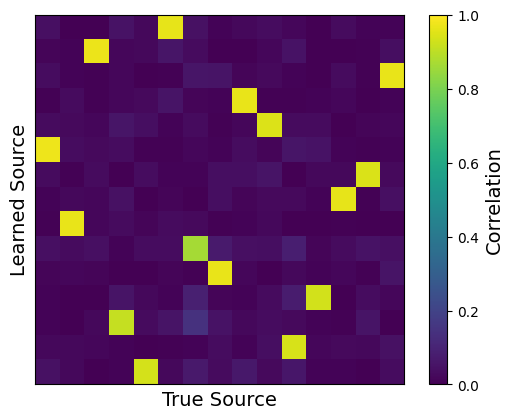}
    \end{subfigure}
    \begin{subfigure}[t]{0.45\textwidth}
        \centering
        \includegraphics[width=\textwidth]{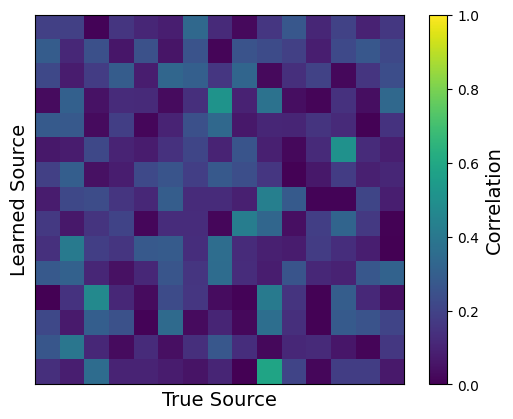}
    \end{subfigure}
    \caption{Correlation plot between learned and ground truth sources for Laplace source (left) and Gamma source (right). The identifiable Laplace model is able to recover the true sources, while the Gamma model is not.}
    \label{fig:tanh_v_lrelu}
\end{figure}

\subsection{Yahoo Stock Returns}

Using yfinance\footnote{https://ranaroussi.github.io/yfinance/}, we downloaded daily closing stock prices for 15 companies for the period 1999-01-22 to 2026-08-31. See appendix for full data and training details. In order to obtain stationary data, we calculated the log returns for each stock, which is given by
\begin{equation*}
    r_s(t) = \log c_s(t) - \log c_s(t-1)
\end{equation*}
for closing price $c_s(t)$ of stock $s$ at time $t$. We then trained Normalizing Flows with $\operatorname{Tanh}$ activations to generate the returns. Because there is no ground truth to validate our results, following \citet{KivvaIdentifiability2022} we trained 10 models starting from distinct random seeds and compared latent sources between models. For all 45 distinct pairs of models we calculated the MCC between the learned sources. 
This was done for several source distributions as reported in \autoref{tab:stock_mcc}.
The results demonstrate that a Laplace prior made training much more stable across seeds, in accordance with \autoref{thm:id_multi_point}.

Additionally, to compare against the results from \citet{KivvaIdentifiability2022} we trained a piecewise affine Normalizing Flow with independent Gaussian Mixture Model source distributions, which we term GMM. Each component had a mixture of two Gaussians that were jointly learned with the Normalizing Flow. We tried LeakyReLU and LeakyHardTanh activations with a 6-layer network. We report the LeakyReLU because it performed slightly better. The poor performance reported in \autoref{tab:stock_mcc} does not disprove the results from \citep{KivvaIdentifiability2022}, but it does demonstrate that this dataset does not satisfy the assumptions of \citep{KivvaIdentifiability2022}.

\begin{table}[tb]
    \centering
    \caption{Average $\pm$ standard error (SE) of MCC between model pairs trained on the Yahoo stock returns dataset for different source distributions. The Laplace distribution is the only one that satisfies Assumption \ref{asm:mult_non_diff_pdf} and thus has the highest average MCC.}
    \label{tab:stock_mcc}
    \begin{tabular}{lc}
        \toprule
        Source Distribution &  Avg MCC $\pm$ SE \\ 
        \midrule
        Laplace & \textbf{0.871 $\pm$ 0.012} \\
        Gaussian & 0.516 $\pm$ 0.003 \\
        Exponential & 0.447 $\pm$ 0.003 \\
        Gamma & 0.492 $\pm$ 0.003 \\
        GMM & 0.480 $\pm$ 0.044 \\
        \bottomrule
    \end{tabular}
\end{table}

\subsection{CelebA}

Lastly, we performed a larger scale experiment on the CelebA dataset \citep{liu2015faceattributes} using the framework provided by \citet{Subramanian2020}.
We trained a VAE with Laplacian prior and a latent space of 128. The encoder and decoder were CNNs with softplus and tanh activations to ensure the network is real analytic. See \autoref{app:celeba_details} for details.

The posterior distribution parameters are provided by $\tilde{\bfg}$. We modeled the posterior with a Laplace distribution since the prior distribution is also a Laplace. For sample $i$ index $j$
\begin{equation*}
p(\bfs^{(i)}_j|\bfx^{(i)}) = \operatorname{Lap}(a^{(i)}_j,b^{(i)}_j) \propto \exp\left(-\frac{|\bfs^{(i)}_j - a^{(i)}_j|}{b^{(i)}_j} \right).
\end{equation*}
\citet{NawaKLD2024} provide the exact form of the KL Divergence between two Laplace distributions,
\begin{equation*}
    D_{KL}\left(\operatorname{Lap}(a_0,b_0) \| \operatorname{Lap}(a_1,b_1)\right) = \log \frac{b_1}{b_0} + \frac{b_0}{b_1} \exp \left(- \frac{|a_0 - a_1|}{b_0} \right) + \frac{|a_0 - a_1|}{b_1} - 1.
\end{equation*}
Following \citep{Bojanowski2017OptimizingTL} we augmented the standard mean squared error (MSE) reconstruction loss with Laplacian Pyramid loss, which helps avoid the excessive blurriness resulting from MSE reconstruction loss. \citet{Bojanowski2017OptimizingTL} gives the Laplacian Pyramid loss as,
\begin{equation*}
\operatorname{Lap}_1(x,x') = \sum_j 2^{2j} |L^j(x) - L^j(x')|_1,
\end{equation*}
where $L_j$ is the $j$-th level of the Laplacian pyramid \citep{LingPyramid2006}. The results shown in \autoref{app:celeba_details} demonstrate that many of the resulting latent variables have a clear interpretation across different images.

\section{Conclusion and Future Work}

Real analytic functions comprise a large class of practically useful functions. Proving identifiability for real analytic transformations of sources with non-differentiable probability density functions is a significant step forward in nICA research. Additionally, unlike many nICA methods \citep{nguyen2025diverseinfluencecomponentanalysis, zheng2022identifiability, gresele2021independent}, RAD does not require expensive Jacobian regularization terms in the objective function. In fact, it can be trained with existing techniques such as Normalizing Flows and Variational Autoencoders.

However, there are still many open problems to be investigated, including the sample complexity of RAD and other nICA methods, and the trade-off between universal approximation and source recovery ability in nICA. Additionally, many nICA methods use VAEs, but assume that the likelihood can be perfectly learned. This is only possible when the KL-divergence between the true posterior and the variational distribution from the encoder is zero. If the encoder has infinite capacity this can happen, but in practice there is a gap between the ELBO and the likelihood. Future work would study how a non-zero gap affects identifiability results.

\section{LLM Usage Disclosure}

While all proofs were written by the authors, an LLM was employed to ideate on proof techniques for some parts of steps 2 and 3 of the proof of \autoref{thm:id_multi_point}. The authors take responsibility for the correctness and comprehensibility of the proofs.
LLMs were also used to provide feedback and suggest improvements for the work.

LLMs were not used to implement any significant portion of the experiments or propose novel ideas. Further LLMs were not used to write the paper apart from the feedback already mentioned.

\section{Reproducibility Statement}

\autoref{thm:id_multi_point}, one of the main contributions, rests on the assumptions explicitly stated in \autoref{sect:main}. The proof of the theorem is contained in \autoref{sect:main} and \autoref{app:step_3}.
To reproduce experiments, we will release the code upon acceptance. Training details are provided in \autoref{app:exp_details} and \autoref{app:celeba_details}.

\bibliography{refs}
\bibliographystyle{iclr2027_conference}
\appendix
\section{Main Theorem Continued Proof Step 3} \label{app:step_3}

We can also decompose $\gamma$ as 
\begin{align*}
    \gamma(u,\bfv) = \gamma_i(u,\bfv) + \sum_{j\neq \pi_1(i)} \log \rho (\bfh_j(u,\bfv)) - \sum_l \log \rho (\bfv_l) \\
    \gamma_i(u,\bfv) = \log \rho(r_i(u,\bfv)(u-\beta_j) + \beta_{\pi_2(j)} ) - \log \rho(u).
\end{align*}
Near $u = \beta_j, \bfs \in \cB_{lm}, \forall (l,m) \neq (i,j)$, $\sum_{l\neq \pi_1(i)} \log \rho (\bfh_l(\bfs)) - \sum_l \log \rho (\bfv_l)$ is real analytic, so $\gamma_i$ must also be real analytic. At this point for notational simplicity, we focus on a fixed $\bfv$ and therefore drop $\bfv$ from the notation. We consider two cases which require slightly different techniques. The cases correspond to when $\bfh_{\pi_1(i)}$ is the identity of $u$ and when $\bfh_{\pi_1(i)}$ reverses $u$.

\textbf{Case 1}: $r_i(u) > 0$
Because $\gamma_i(u)$ is real analytic, $\gamma_i'(u)$ must be continuous across $u=\beta_j$. 
\begin{align*}
\lim_{u\rightarrow \beta_j^+}\gamma_i'(u) 
= \lim_{u\rightarrow \beta_j^+} \left(\alpha_{\pi_2(j)}(r_i(u)(u-\beta_j) + \beta_{\pi_2(j)}) - \alpha_{j}(u)\right)' \\
= \lim_{u\rightarrow \beta_j^+} \alpha_{\pi_2(j)}'(r_i(u)(u-\beta_j) + \beta_{\pi_2(j)})(r_i'(u)(u-\beta_j) + r_i(u)) - \alpha_{j}'(u) \\
= \alpha_{\pi_2(j)}'(\beta_{\pi_2(j)})r_i(\beta_j) - \alpha_{j}'(\beta_j)
\end{align*}
The same reasoning can be done for $u \rightarrow \beta_j^-$. Continuity at $u=\beta_j$ gives,
\begin{align*}
\alpha_{\pi_2(j)}'(\beta_{\pi_2(j)})r_i(\beta_j) - \alpha_{j}'(\beta_j)
 = \alpha_{\pi_2(j)-1}'(\beta_{\pi_2(j)})r_i(\beta_j) - \alpha_{j-1}'(\beta_j)
\end{align*}
\begin{equation} \label{eq:r_ps_const}
r_i(\beta_j) = \frac{\alpha_{j}'(\beta_j) - \alpha_{j-1}'(\beta_j)}{\alpha_{\pi_2(j)}'(\beta_{\pi_2(j)}) - \alpha_{\pi_2(j)-1}'(\beta_{\pi_2(j)})}
\end{equation}
But since $\frac{\partial \bfh_l(\beta_j)}{\partial u} = r_i'(\beta_j)(\beta_j-\beta_j) + r_i(\beta_j) > 0$, $\pi_2$ is the identity from the continuity argument in step 2. Therefore $r_i(\beta_j)=1$.

Let the power series of $r_i$ around $u=\beta_j$ be $r_i(u) = 1+\sum_{l=n}^\infty c_l \left(u-\beta_j\right)^l$, where coefficients $c_1,\dots, c_{n-1}=0$ for $n \geq 1$, where $n=1$ means no coefficients are $0$.
Since $\alpha_j$ is real analytic it can be described by its power series $\alpha_j(x)= \sum_{m=1}^\infty a_m^{(j)} (x-\beta_j)^m$. Therefore
\begin{align*}
\alpha_j(\beta_j + (u-\beta_j)r_i(u)) = 
\sum_{m=1}^\infty a_m^{(j)} \left((u - \beta_j) + \sum_{l=n}^\infty c_l (u-\beta_j)^{l+1} \right)^m \\
\left(u - \beta_j + \sum_{l=n}^\infty c_l (u-\beta_j)^{l+1} \right)^m =  (u-\beta_j)^m + m c_n (u-\beta_j)^{m + n} + \cO((u-\beta_j)^{m + n + 1}) \\
\alpha_j(\beta_j + (u-\beta_j)r_i(u)) - \alpha_j(u) = a_1^{(j)} c_n (u-\beta_j)^{n+1} + \cO((u-\beta_j)^{n+2})
\end{align*}
We can do the same thing for $\alpha_{j-1}$ to get 
\begin{equation*}
\alpha_{j-1}(\beta_j + (u-\beta_j)r_i(u)) - \alpha_{j-1}(u) = a_1^{(j-1)} c_n (u-\beta_j)^{n+1} + \cO((u-\beta_j)^{n+2})
\end{equation*}

Because $\gamma_i(u)$ is real analytic, as $u \rightarrow \beta_j$ the coefficients of $(u-\beta_j)^{n+1}$ must be equal. But, $a_1^{(j-1)}\neq a_1^{(j)}$, so $c_n=0$. We can do this inductively to show that all Taylor coefficients of $r_i$ are $0$ except for the constant term. Since $r_i$ is real analytic, this means that $r_i \equiv 1$.

\textbf{Case 2:} $r_i(u) < 0$

\begin{figure}
    \centering
    \includegraphics[width=0.5\linewidth]{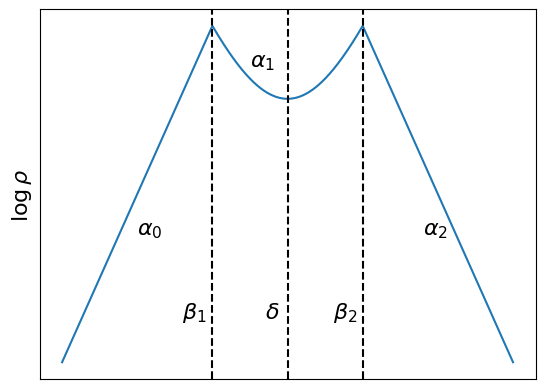}
    \caption{Example Log Density}
    \label{fig:even_log_density}
\end{figure}

When $d$ is odd, for $j = \frac{d+1}{2}$ $\pi_2(j) = j$ due to the fact that $\pi_2$ is a reverse ordering. 
The rest of the arguments mirror Case 1, resulting in $r_i \equiv -1$. This means that order reversal can only happen on symmetric densities.

However, for $d$ even, the argument does not mirror exactly. Because $r_i \neq 0$ and $r_i$ is continuous, it must have the same sign on its domain.
There exists $\delta \in [\beta_{\frac{d}{2}},\beta_{\frac{d}{2}+1}]$, such that $\bfh_{\pi_1(i)}(\delta, \bfv) = \delta$ (see \autoref{fig:even_log_density}). So using the Weierstrass Preparation Theorem as before we factorize $\bfh_{\pi_1(i)}$
\begin{equation*}
\bfh_{\pi_1(i)}(\bfs)- \delta = r_i(\bfs)(\bfs_i - \delta)
\end{equation*}
The power series of $r_i$ and $\alpha_{j}$ are $r_i(u) = c_0+\sum_{l=n}^\infty c_l \left(u-\delta\right)^l$ and $\alpha_{j}(x)=\sum_{m=1}^\infty a_m^{(j)} (x-\delta)^m$.
\begin{align*}
\alpha_{\pi_2(j)}(r_i(u)(u-\delta) + \delta) = \sum_{m=1}^\infty a_m^{(\pi_2(j))} \left(c_0(u-\delta) + \sum_{l=n}^\infty c_l \left(u-\delta\right)^{l+1}\right)^m \\
\left(c_0(u-\delta) + \sum_{l=n}^\infty c_l \left(u-\delta\right)^{l+1}\right)^m = c_0^m(u-\delta)^m + mc_0^{m-1} c_n \left(u-\delta\right)^{n+m} + \cO\left(\left(u-\delta\right)^{n+m+1}\right) \\
\alpha_{\pi_2(j)}(r_i(u)(u-\delta) + \delta) - \alpha_{j-1}(u)\\
= \sum_{m=1}^\infty (a_m^{(\pi_2(j))}c_0^m - a_m^{(j-1)})(u-\delta)^m + 
a_1^{(\pi_2(j))} c_n \left(u-\delta\right)^{n+1} + \cO\left(\left(u-\delta\right)^{n+2}\right)
\end{align*}

This is precisely $\gamma_i$ for $u\in (\beta_{j-1}, \beta_j)$. For each interval we can do this same expansion. Since $\gamma_i$ is real analytic, any power series expansion around $\delta$ must have the same coefficients. From the coefficients of the first order terms we have,

\begin{align*}
a_1^{(\pi_2(j))}c_0 - a_1^{(j-1)} = a_1^{(\pi_2(j)-1)}c_0 - a_1^{(j)} \\
c_0 = \frac{a_1^{(j)} - a_1^{(j-1)}}{a_1^{(\pi_2(j)-1)} - a_1^{(\pi_2(j))}}
\end{align*}

With doing this same calculation for the intervals $(\beta_{\pi_2(j)-1}, \beta_{\pi_2(j)})$ and $(\beta_{\pi_2(j)}, \beta_{\pi_2(j)+1})$ we arrive at

\begin{align*}
c_0 = \frac{a_1^{(\pi_2(j)-1)} - a_1^{(\pi_2(j))}}{a_1^{(j)} - a_1^{(j-1)}} \\
c_0 = \frac{1}{c_0} \implies c_0=-1
\end{align*}

because $c_0$ cannot be positive due to the case $r_i < 0$. The $n+1$ order term must also have equal coefficients for all intervals. Thus the $(\beta_j, \beta_{j+1})$ interval produces the $n+1$ order taylor coefficient for $\gamma_i$,

\begin{equation} \label{eq:n_coef}
a_{n+1}^{(\pi_2(j)-1)} (-1)^{n+1} - a_{n+1}^{(j)} + a_1^{(\pi_2(j)-1)} c_n
\end{equation}

If $n$ is odd we equate the $n+1$ coefficient for intervals $(\beta_{j-1},\beta_{j})$ and $(\beta_j,\beta_{j+1})$ and then for intervals $(\beta_{\pi_2(j)-1},\beta_{\pi_2(j)})$ and $(\beta_{\pi_2(j)},\beta_{\pi_2(j)+1})$ to get the following two expressions for $c_n$,

\begin{align*}
c_n = \frac{(a_{n+1}^{(\pi_2(j)-1)} - a_{n+1}^{(\pi_2(j))}) - (a_{n+1}^{(j)} - a_{n+1}^{(j-1)})}{a_1^{(\pi_2(j))} - a_1^{(\pi_2(j)-1)}} \\
c_n = \frac{(a_{n+1}^{(j)} - a_{n+1}^{(j-1)}) - (a_{n+1}^{(\pi_2(j)-1)} - a_{n+1}^{(\pi_2(j))})}{a_1^{(j)} - a_1^{(j-1)}} 
\end{align*}

For the kinks in $\rho$ and $\rho \circ \bfh$ to cancel each other when $r_i < 0$, the differences $a_1^{(\pi_2(j))} - a_1^{(\pi_2(j)-1)}$ and $a_1^{(j)} - a_1^{(j-1)}$ must have the same sign. However the numerators in the equation above are negatives of each other. Therefore $c_n = 0$.

If $n$ is even we equate the coefficients in \autoref{eq:n_coef} for intervals $(\beta_{j-1}, \beta_j)$ and $(\beta_{\pi_2(j)}, \beta_{\pi_2(j)+1})$ where $j < \frac{d}{2}$ to get,


\begin{align*}
c_n = \frac{(a_{n+1}^{(j-1)} - a_{n+1}^{(\pi_2(j))}) (-1) - (a_{n+1}^{(\pi_2(j))} - a_{n+1}^{(j-1)})}{a_1^{(\pi_2(j))} - a_1^{(j-1)}} = 0
\end{align*}

$a_1^{(\pi_2(j))} - a_1^{(j-1)} \neq 0$ by the cancellation argument above.

\section{Experiment Details} \label{app:exp_details}

\subsection{Synthetic}
For the first experiment with $m=k=5,10$ we generated 10,000 data points to train on for 500 epochs. 
For $m=k=15,25$ we generated 5,000 data points to train on for 1000 epochs.
The optimizer was Adam with learning rate 5e-3 and batches of 128 samples. We employed gradient clipping for training stability.

We experimented with two other generating functions. Mixture C was a 3-layer neural network with orthogonal layers and $\operatorname{Tanh}$ activation.
Mixture D was an alternate post-nonlinear formulation with $a_i(y) = w_i y + \tanh (y)(1-w_i)$.

\begin{table}[tb]
    \centering
    \caption{Average MCC $\pm$ standard deviation for Laplace source model with different activation functions in the Normalizing Flow. Real analytic activation functions achieved substantially higher MCCs than non-real analytic activation functions.}
    \label{tab:act_fun_comp_CD}
    \begin{tabular}{lcccc}
        \toprule
        & \multicolumn{2}{c}{$k=5$} & \multicolumn{2}{c}{$k=10$} \\
        \cmidrule(lr){2-3} \cmidrule(lr){4-5}
        Activation Function & Mixture A & Mixture B & Mixture A & Mixture B \\
        \midrule
        Tanh & 0.931 $\pm$ 0.015 & 0.993 $\pm$ 0.003 & 0.928 $\pm$ 0.008 & 0.983 $\pm$ 0.006 \\
        Sigmoid & 0.926 $\pm$ 0.016 & 0.987 $\pm$ 0.002 & 0.860 $\pm$ 0.037 & 0.987 $\pm$  0.006 \\
        Softplus & 0.877 $\pm$ 0.063 & 0.991 $\pm$ 0.002 & 0.829 $\pm$ 0.042 & 0.980 $\pm$ 0.007 \\
        LeakyHardTanh & 0.656 $\pm$ 0.080  & 0.660 $\pm$ 0.154 & 0.561 $\pm$ 0.092 & 0.540 $\pm$ 0.039 \\
        LeakyReLU & 0.136 $\pm$ 0.116 & 0.066 $\pm$ 0.041 & 0.095 $\pm$ 0.105 & 0.038 $\pm$ 0.011 \\
        \bottomrule
    \end{tabular}
\end{table}

\subsection{Stock NF}

The stock data was from the period 1999-01-22 to 2026-08-31. We chose large companies from a wide variety of market sectors (\autoref{tab:companies}). From \autoref{fig:Stock_EDA} the auto-correlation does not seem to be significant and the data appears to be relatively stationary. For each model we trained for 750 epochs with Adam optimizer, learning rate set at 1e-2, with a decline on plateau scheduler, and batches of 256 samples. We arbitrarily chose the seeds 9-18 to initialize the 10 models with.

\begin{table}[htb]
    \centering
    \caption{Company names of stock data in the Yahoo Stock experiment. For easier replication we provide the Yahoo Finance Symbol over the company names.}
    \label{tab:companies}
    \begin{tabular}{c|c|c|c}
        AAPL &  MSFT & NVDA & AMZN \\
        Apple Inc. & Microsoft Corp. & NVIDIA Corp. & Amazon Inc. \\
        \hline
        ALL & JPM & PFE & JNJ \\
        The Allstate Corp. & JPMorgan Chase \& Co. & Pfizer Inc. & Johnson \& Johnson \\ \hline
        CVX & WMT & CAT & DIS \\
        Chevron Corp. & Walmart Inc. & Caterpillar Inc. & The Walt Disney Co. \\
        \hline
        XOM & AMD & BRK-A & \\
        ExxonMobil Corp. & Advanced Micro Devices Inc. & Berkshire Hathaway Inc.& \\
    \end{tabular}
\end{table}

\begin{figure}[bht]
    \centering
    \begin{subfigure}[t]{0.45\textwidth}
        \centering
        \includegraphics[width=\textwidth]{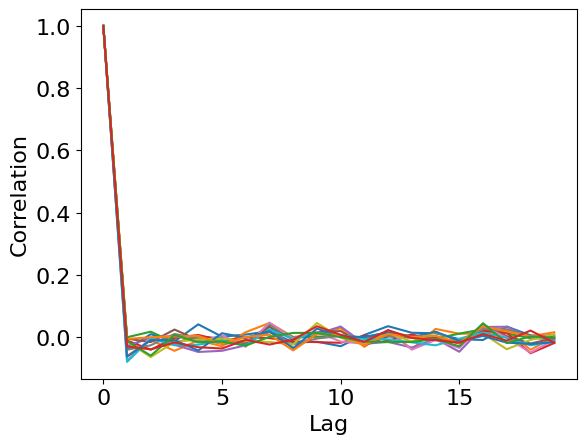}
    \end{subfigure}
    \begin{subfigure}[t]{0.45\textwidth}
        \centering
        \includegraphics[width=\textwidth]{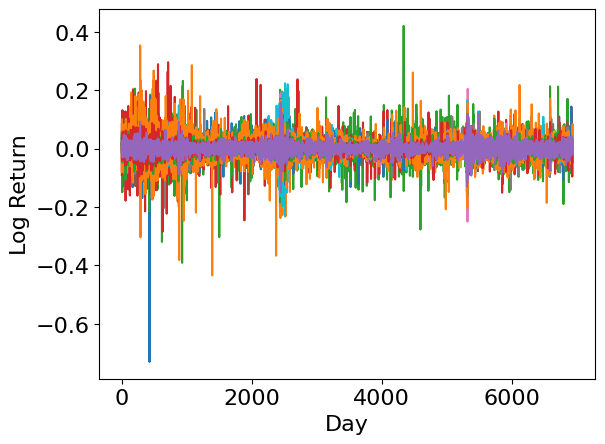}
    \end{subfigure}
    \caption{Auto-correlation and trace plots of stock returns data}
    \label{fig:Stock_EDA}
\end{figure}


\section{CelebA} \label{app:celeba_details}

The Variational Autoencoder architecture used convolutional layers for the encoder and convolutional transpose layers for the decoder. The kernel size was 3, with 32, 64, 128, 256, 512 filters for the 5 convolutional layers. We used the $\operatorname{Softplus}$ activation for hidden layers and the $\operatorname{Tanh}$ for the final layer of the decoder. We used a latent dimension of $k=128$ and trained for 100 epochs with a learning rate of $0.005$ on one B200 GPU.

\begin{figure}[ht]
    \centering
    \begin{subfigure}[t]{0.95\textwidth}
        \centering
        \includegraphics[width=\textwidth]{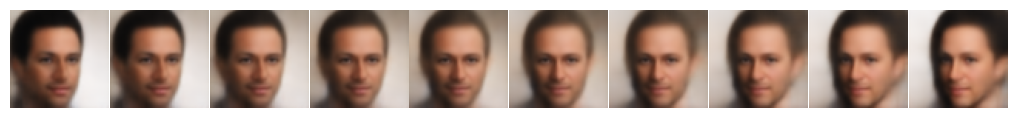}
    \end{subfigure}
    \begin{subfigure}[t]{0.95\textwidth}
        \centering
        \includegraphics[width=\textwidth]{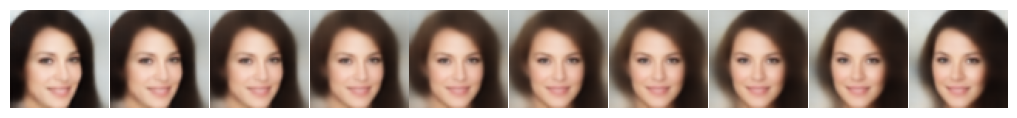}
    \end{subfigure}
    \begin{subfigure}[t]{0.95\textwidth}
        \centering
        \includegraphics[width=\textwidth]{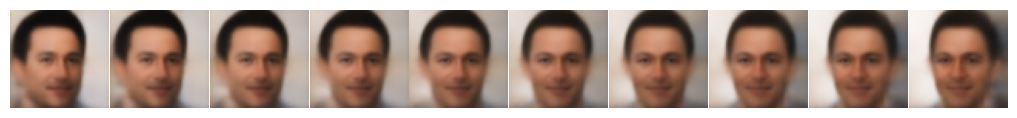}
    \end{subfigure}
    \begin{subfigure}[t]{0.95\textwidth}
        \centering
        \includegraphics[width=\textwidth]{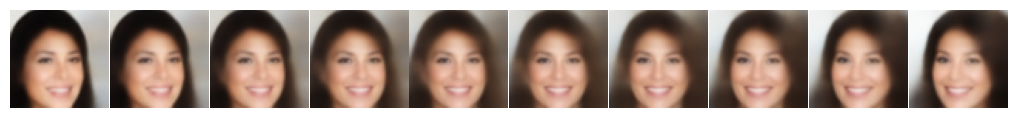}
    \end{subfigure}
    \caption{Head orientation}
    \label{fig:celebA_l34}
\end{figure}

\begin{figure}[ht]
    \centering
    \begin{subfigure}[t]{0.95\textwidth}
        \centering
        \includegraphics[width=\textwidth]{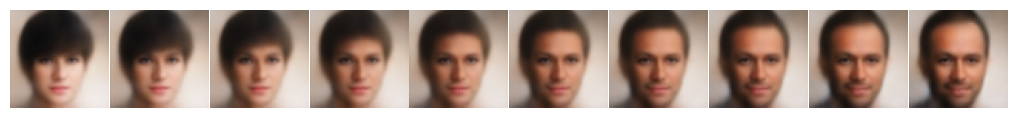}
    \end{subfigure}
    \begin{subfigure}[t]{0.95\textwidth}
        \centering
        \includegraphics[width=\textwidth]{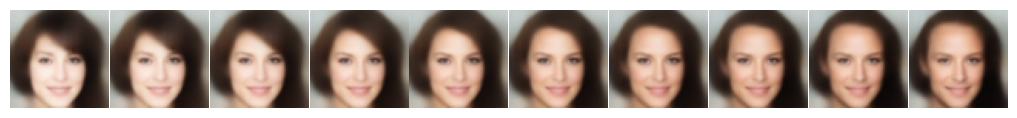}
    \end{subfigure}
    \begin{subfigure}[t]{0.95\textwidth}
        \centering
        \includegraphics[width=\textwidth]{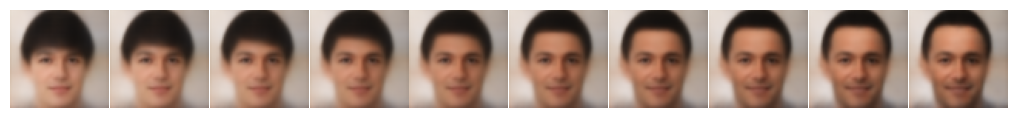}
    \end{subfigure}
    \begin{subfigure}[t]{0.95\textwidth}
        \centering
        \includegraphics[width=\textwidth]{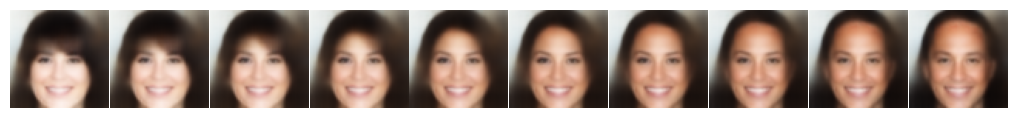}
    \end{subfigure}
    \caption{Bangs}
    \label{fig:celebA_l41}
\end{figure}

\begin{figure}[ht]
    \centering
    \begin{subfigure}[t]{0.95\textwidth}
        \centering
        \includegraphics[width=\textwidth]{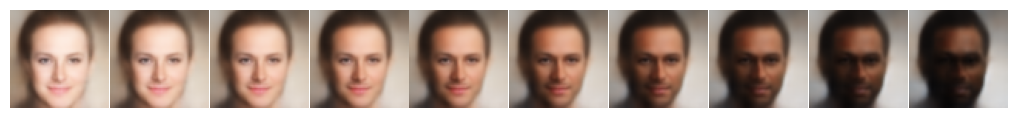}
    \end{subfigure}
    \begin{subfigure}[t]{0.95\textwidth}
        \centering
        \includegraphics[width=\textwidth]{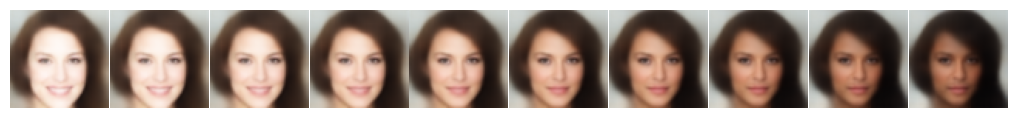}
    \end{subfigure}
    \begin{subfigure}[t]{0.95\textwidth}
        \centering
        \includegraphics[width=\textwidth]{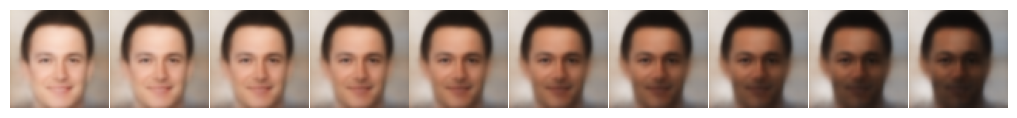}
    \end{subfigure}
    \begin{subfigure}[t]{0.95\textwidth}
        \centering
        \includegraphics[width=\textwidth]{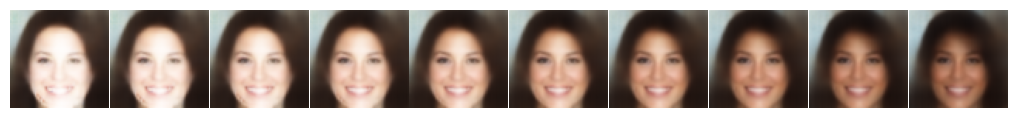}
    \end{subfigure}
    \caption{Skin Color}
    \label{fig:celebA_l82}
\end{figure}

\begin{figure}[ht]
    \centering
    \begin{subfigure}[t]{0.95\textwidth}
        \centering
        \includegraphics[width=\textwidth]{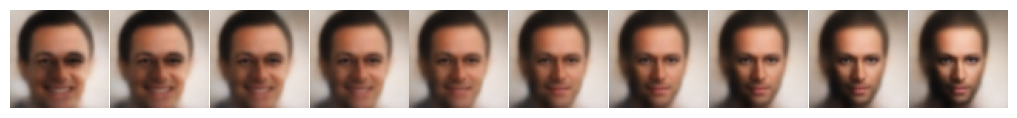}
    \end{subfigure}
    \begin{subfigure}[t]{0.95\textwidth}
        \centering
        \includegraphics[width=\textwidth]{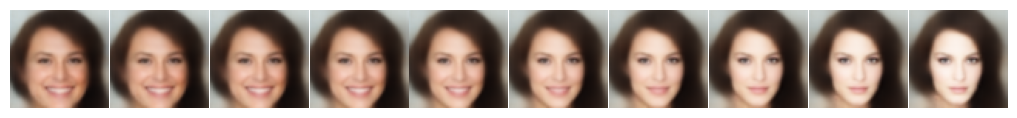}
    \end{subfigure}
    \begin{subfigure}[t]{0.95\textwidth}
        \centering
        \includegraphics[width=\textwidth]{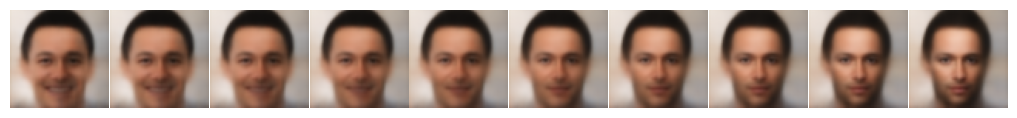}
    \end{subfigure}
    \begin{subfigure}[t]{0.95\textwidth}
        \centering
        \includegraphics[width=\textwidth]{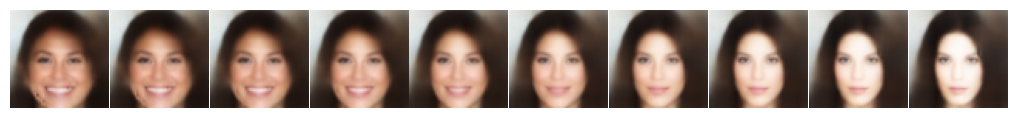}
    \end{subfigure}
    \caption{Smile/Broad face}
    \label{fig:celebA_l83}
\end{figure}

\end{document}